\documentclass{llncs}

\usepackage{graphicx}
\usepackage{amsmath}
\usepackage{amssymb}

\usepackage{amssymb}
\usepackage{amsmath}
\usepackage{xspace}
\usepackage{tabularx}
\usepackage{mathptmx}
\usepackage{microtype}
\usepackage{verbatim}
\usepackage{multicol}

\newcommand{\PROP}{\ensuremath{\mathrm{PROP}}\xspace}

\newcommand{\LIT}{\ensuremath{\mathrm{Lit}}\xspace}

\newcommand{\LABELS}{\ensuremath{\mathrm{Lab}}\xspace}

\newcommand{\OBL}{\ensuremath{\mathbf{OBL}}\xspace}

\newcommand{\To}{\Rightarrow}

\newcommand{\non}{\ensuremath{\mathord{\sim}}}
\newcommand{\defeater}{\leadsto}
\newcommand{\nmdash}{\mathbin{\mathchoice{\vert{\mkern-8mu}\sim}
{\vert{\mkern-8mu}\sim}
{\scriptstyle\vert{\mkern-5mu}\scriptstyle\sim}
{\scriptscriptstyle\vert{\mkern-4mu}\scriptscriptstyle\sim}}}
\newcounter{clause}

\newcommand{\set}[1]{\{#1\}}

\newcommand{\logic}[1]{\gmath{\mathbf{#1}}}

\newcommand{\gmath}[1]{\ensuremath{#1}\xspace}
\newcommand{\seq}[2][n]{\gmath{#2_{1},\dots,#2_{#1}}}

\newcommand{\Lg}[1]{\gmath{\mathsf{#1}}}

\newcommand{\tuple}[1]{\gmath{\langle #1\rangle}}
\newcommand{\tset}[1]{\ensuremath{||#1||}}

\newcommand{\cF}{\gmath{\mathcal{F}}}
\newcommand{\cS}{\gmath{\mathcal{S}}}

\newcommand{\cM}{\gmath{\mathcal{M}}}
\newcommand{\cN}{\gmath{\mathcal{N}}}

\newcommand{\PERM}{\ensuremath{\mathbf{PERM}}\xspace}

\newtheorem{intuition}{Intuition}

\newcommand{\GSConc}{\ensuremath{\mathsf{Conc}}\xspace}

\newcommand{\GSSub}{\ensuremath{\mathsf{Sub}}\xspace}

\newcommand{\GSTopRule}{\ensuremath{\mathsf{TopRule}}\xspace}

\newcommand{\lit}{\phi}

\newcommand{\AR}{\mathcal{A}} 

\newcommand{\GSRules}{\ensuremath{\mathsf{Rules}}\xspace}

\newcommand{\defeat}{\gg}

\newcommand{\Case}[2]{\mathsf{Expl}(#1,#2)}

\newcommand{\TO}{\rightarrow}
\newcommand{\FACTS}{\ensuremath{F}}

\newcommand{\str}{s}

\newcommand{\Arg}{\ensuremath{\mathsf{arg}}}
\newcommand{\ARG}{\ensuremath{\mathsf{AF}}}

\begin{document}
\title{Stable Normative Explanations: From Argumentation to Deontic Logic}
\author{Cecilia Di Florio\inst{1}
\and
Guido Governatori\inst{2}\orcidID{0000-0002-9878-2762} \and \\
Antonino Rotolo\inst{1}\orcidID{0000-0001-5265-0660} \and
Giovanni Sartor\inst{3}\orcidID{0000-0001-5680-0080} 
}
\authorrunning{Cecilia Di Florio et al.}
%
\institute{ALMA AI and Department of Legal Studies, University of Bologna, Italy \
\email{\{ cecilia.diflorio2,antonino.rotolo\}@unibo.it}
\and
Cooroibah, QLD 4565, Australia \
\email{guido@governatori.net} 
\and
ALMA AI/Department of Legal Studies,  University of Bologna, and EUI, Italy 
\email{giovanni.sartor@unibo.it}
}
\maketitle   
  \begin{abstract}
This paper examines how a notion of stable explanation developed elsewhere in Defeasible Logic can be expressed in the context of formal argumentation. With this done, we discuss the deontic meaning of this reconstruction and show how to build from argumentation neighborhood structures for deontic logic where this notion of explanation can be characterised. Some direct complexity results are offered.
  \end{abstract}


\section{Introduction}\label{sec:intro}
Resorting to machine learning to predict outcomes in legal proceedings is very much discussed in the literature as well as by policy-makers (for an overview, see, e.g., \cite{Medvedeva2020-MEDUML,BexP21,ATKINSON2020103387}). Indeed, such techniques can be used for algorithmic decision predictors to support judges in individual cases, to assist litigants in estimating their likelihood of winning a case or in examining various biases on legal decision-making processes \cite{BexP21}. One of the most challenging contexts in which to introduce AI is within courts. Judges are often reluctant to adopt these tools for two reasons: (a) it could undermine the independent exercise of judicial power, and (b) AI is anything but transparent and explainable.


Developing Explainable-AI systems is thus more and more important in the law since `\emph{transparency}' and `\emph{justification}' of legal decision-making both require formalising normative explanations \cite{AkataBRDDEFGHHH20}. Normative explanation is a type of explanation where norms (in addition to factual information) are crucial: if reframed in the context of legal decision-making, this means to explain why a legal conclusion (such as an obligation) ought to be the case on the basis of certain norms (such as one prescribing to compensate for the damages for which we are liable) and facts (such as the fact that I causally contributed to cause a damage)  \cite{Alexy1989TL,Peczenik89}. In the context of judicial reasoning, the  idea of normative explanation is now emerging in the literature (see \cite{GovernatoriORC22,PrakkenR22,LiuLRS22,LiaoT20}).

Legal proceedings are adversarial in nature. In this perspective, if a judge or a litigant aim at predicting possible outcomes, this fact must be taken into account, and formal tools to make such predictions understandable should allow for checking if a certain legal outcome is \emph{stable} \cite{lpnmr2022,GovernatoriORC22,OdekerkenBBT22}. This is especially true in an argumentation perspective, where the adversarial structure of proceedings become more transparent. In such a perspective, given some facts, the proceeding aims at determining what legal requirements hold, and whether such legal requirements have been fulfilled.  (In)Stability means that, if more/new facts were presented, the outcome of a case might be quite different or can even be modified. How to ensure a specific outcome for a case, which, in an adversarial setting, can be understood as addressing the question of how to ensure that the facts presented by a party are `resilient' to the attacks from the opponent? 
 
In this paper we adopt \cite{lpnmr2022,GovernatoriORC22}'s definition of stability and elaborate it in the argumentation setting of Defeasible Logic \cite{Antoniou2001255}. Apart from some details, while valuable, this extension is technically rather straightforward. However, we are interested  in second, and more challenging, research question: \emph{What is the deontic meaning of stable normative explanation as developed in an argumentation setting}? In fact, in legal argumentation, a typical outcome of judicial decisions are obligations and permissions.

In moving to the deontic domain, we must notice that deontic argumentation can be developed in various ways \cite{GovernatoriRR18,RiveretRS19}. As commonly done in the AI\&Law literature \cite{PrakkenS15}, we assume that legal norms are rules having the form $\seq{\lit} \To \psi$ and we follow this intuition:

\begin{intuition}
    Let $\ARG$ be an argumentation framework where arguments are built using rules of the form $\seq{\lit} \To \psi$. Then 
     $\OBL \psi$ holds in $\ARG$ iff $\psi$ is justified w.r.t. $\ARG$.  
\end{intuition}

Once we have defined the argumentative setting and identified some notions of normative explanation, we adapt \cite{GovernatoriRC12}'s method  and show how \emph{this machinery can be reconstructed in neighborhood semantics for classical deontic logics \cite{chellas:1980} and how the notion of explanation can be semantically characterised}.

The layout of article is as follows. Section \ref{sec:dung} recalls the basics of Defeasible Logic and offers a variant of the idea of  argumentation framework based on such a logic. Section \ref{sec:TypesArgExpl} presents the definitions of normative explanation and stable normative explanation. Section \ref{sec:FromArgToModal} illustrates how to move from argumentation structures to neighbourhood semantics for deontic logic. Section  \ref{sec:deontic_explanation} applies the ideas of Sections \ref{sec:TypesArgExpl} and \ref{sec:FromArgToModal} to semantically reconstruct the concept of normative explanation. 

\section{Background}
\label{sec:dung}

\subsection{Defeasible Logic}\label{sec:DL}

The logical apparatus we utilise  is the standard Defeasible Logic (DL) \cite{Antoniou2001255}. In this section we present the basics of DL.


Let $\PROP$ be the set of propositional atoms, then the set of literals $\LIT = \PROP \cup \{\neg p\, |\, p \in \PROP\}$. The \emph{complementary} of a literal $p$ is denoted by $\non p$: if $p$ is a positive literal $q$ then $\non p$ is $\neg q$, if $p$ is a negative literal $\neg q$ then $\non p$ is $q$. Literals are denoted by lower-case Roman letters. Let $\LABELS$ be a set of labels to represent names of rules.

A \emph{defeasible theory} $D$ is a tuple $(\FACTS, R, >)$, where $\FACTS$ is the set of facts (indisputable statements), $R$ is the rule set, and $>$ is a binary relation over $R$.

$R$ is partitioned into three distinct sets of rules, with different meanings to draw different ``types of conclusions''. \emph{Strict rules} are rules in the classical sense: whenever the premises are the case, so is the conclusion. We then have \emph{defeasible rules} which represent the non-monotonic part (along which defeaters) of the logic: if the premises are the case, then typically the conclusion holds as well unless we have contrary evidence that opposes and prevents us from drawing such a conclusion. Lastly, we have \emph{defeaters}, which are special rules whose purpose is to prevent contrary evidence from being the case. It follows that in DL, through defeasible rules and defeaters,  we can represent in a natural way exceptions (and exceptions to exceptions, and so forth). 

We finally have the superiority relation $>$, a binary relation among couples of rules that is the mechanism to solve conflicts. Given the two rules $r$ and $t$, we have $\langle r, t\rangle \in >$ (or simply $r > t$), in the scenario where both rules may fire (can be activated), $r$'s conclusion will be preferred to $t$'s.


In general, a rule $r \in R$ has the form $r\colon A(r) \hookrightarrow C(r)$, where: (i) $r \in \LABELS$ is the unique name of the rule, (ii) $A(r) \subseteq \LIT$ is $r$'s (set of) antecedents, (iii) $C(r) = l \in \LIT$ is its conclusion, and (iv) $\hookrightarrow\in \set{\TO, \To, \defeater}$ defines the type of rule, where: $\TO$ is for strict rules, $\To$ is for defeasible rules, and $\defeater$ is for defeaters.

Some standard abbreviations. $ R_s$ denotes the set of strict rules in $R$, and the set of strict and defeasible rules is denoted by $R\str$; $R[l]$ denotes the set of all rules whose conclusion is $l$. 

A \emph{conclusion} of $D$ is a \emph{tagged literal} with one of the following forms:

\begin{description}
	\item[$+\Delta l$ ({\normalfont resp.} $-\Delta l$)] means that $l$ is \emph{definitely proved} (resp. \emph{strictly refuted/non provable}) in $D$, i.e., there is a definite proof for $l$ in $D$ (resp. a definite proof does not exist).
	
	\item[$+\partial l$ ({\normalfont resp.} $-\partial l$)] means that $l$ is \emph{defeasibly proved} (resp. \emph{defeasibly refuted}) in $D$, i.e., there is a defeasible proof for $l$ in $D$ (resp. a definite proof does not exist).
\end{description}

The definition of proof is also the standard in DL. Given a defeasible theory $D$, a proof $P$ of length $n$ in $D$ is a finite sequence $P(1), P(2), \dots, P(n)$ of tagged formulas of the type $+\Delta l$, $-\Delta l$, $+\partial l$, $-\partial l$, where the proof conditions defined in the rest of this section hold. $P(1..n)$ denotes the first $n$ steps of $P$. 

All proof tags for literals are standard in DL \cite{Antoniou2001255}. We present only the positive ones as the negative proof tags can be straightforwardly obtained by applying the \emph{strong negation principle} to the positive counterparts. The strong negation principle applies the function that simplifies a formula by moving all negations to an innermost position in the resulting formula, replaces the positive tags with the respective negative tags, and the other way around, see \cite{DBLP:journals/igpl/GovernatoriPRS09}.

Positive proof tags ensure that there are effective decidable procedures to build proofs; the strong negation principle guarantees that the negative conditions provide a constructive method to verify that a derivation of the given conclusion is not possible.

%

%
The definitions of $\pm\Delta$ describe forward-chaining of strict rules and are omitted. 

Defeasible derivations are based on the notions of a rule being applicable or discarded. A rule is \emph{applicable} at a given derivation step when every antecedent has been proved at any previous derivation step. Symmetrically, a rule is \emph{discarded} when at least one antecedent has been previously refuted.

\begin{definition}[Applicable \& Discarded]\label{def:ApplDisc}\

Given a defeasible theory $D$, a literal $l$, and a proof $P(n)$, we say that

\begin{itemize}
	\item $r \in R[l]$ is \emph{applicable} at $P(n+1)$ iff $\forall a \in A(r).\, +\partial a \in P(1..n)$.

	\item $r \in R[l]$ is \emph{discarded} at $P(n+1)$ iff $\exists a \in A(r).\, -\partial a \in P(1..n)$.
\end{itemize}
 
\end{definition}
Note that a strict rule can be used to derive defeasible conclusions when it is applicable and at least one of its premises is defeasibly but not strictly proved.

\begin{definition}[$+\partial$]\label{def:+partial}\

\begin{tabbing}
  $+\partial l$: \=If $P(n+1)=+\partial l$ then either\+\\
  (1) \= $+\Delta l \in P(1..n)$, or\\  
  (2.1) \= $-\Delta \non l \in P(1..n)$, and\\
  (2.2) \= $\exists r \in R[l]$ applicable s.t.\\
  (2.3) \= $\forall s \in R[\non l]$ either \+\\
  		(2.3.1) \= $s$ discarded, or\\
 		(2.3.2) \= $\exists t \in R[l]$ applicable s.t. $t > s$.
\end{tabbing}  
\end{definition}
A literal is defeasibly proved if (1) it has already proved as a strict conclusion, or (2.1) the opposite is not and (2.2) there exists an applicable, defeasible or strict, rule such that any counter-attack is either (2.3.1) discarded or (2.3.2) defeated by an applicable, stronger rule supporting $l$. Note that, whereas $s$ and $t$ may be defeaters, $r$ \emph{may not}, as we need a strict or defeasible, applicable rule to draw a conclusion.


The last notions introduced in this section are those of extension of a defeasible theory. Informally, an extension is everything that is derived and disproved.

\begin{definition}[Theory Extension]
Given a defeasible theory $D$, we define the set of positive and negative conclusions of $D$ as its \emph{extension}:
$E(D) = (+\Delta, -\Delta, +\partial, -\partial )$, where $\pm\# = \{l |\, l$ appears in $D$ and $D \vdash \pm\# l \}$, $\# \in \{\Delta, \partial\}$. 
\end{definition}


%

\begin{theorem}\label{th:standard-complexity}\cite{maher:complexity}
Given a defeasible theory $D$, its extension $E(D)$ can be computed in time polynomial to the size of the theory. 
\end{theorem}

\subsection{Argumentation in Defeasible Logic}\label{sec:Arg-DL}

Argumentation frameworks for DL and the corresponding argumentation semantics have been in general studied in  \cite{GovernatoriMAB04}. 
Here, we present a variant of it, which is based on a fragment of DL without strict rules and defeaters.
Also, since rules are meant to express norms, facts have a special status here, i.e.---as argued in \cite{lpnmr2022,GovernatoriORC22}---they are meant to capture purely factual information and thus do not occur in the heads of rules (which are supposed to lead to normative conclusions).



\begin{definition}\label{def:theory}
An \emph{argumentation theory} $D$ is a defeasible theory $(\FACTS,R,>)$ 
where 
\begin{itemize} 
  \item $R$ is a (finite) set of defeasible rules,
   \item \sloppy $\FACTS\subseteq \LIT$ is a finite consistent set of facts where, for each $p\in\FACTS$, $R[p]\cup R[\non p]=\emptyset$, and
  \item $>\subseteq R\times R$ is a superiority relation on $R$.
 \end{itemize}
\end{definition}

%
By combining the rules in a theory, we can build arguments (we adjust the definition in \cite{Prakken10} to meet Definition \ref{def:theory}). 
In what follows, for a given argument $A$, \GSConc returns its conclusion, \GSSub returns all its sub-arguments, \textsf{Rules} returns all the  rules in the argument and, finally, \GSTopRule returns the last inference rule  in the argument. 

\begin{definition}[Argument]\label{def:argument} Let $D= (\FACTS, R, >)$ be an argumentation theory. 
An argument $A$ constructed from $D$ has either the the form $\To_{\FACTS} \lit$ (\emph{factual argument}), where $\lit\in \FACTS$, or the form 
$A_1, \ldots , A_n \To_r \lit$ (\emph{plain argument}), where $1\leq k \leq n$, and 
  \begin{itemize}
  \item $A_k$ is an argument constructed from $D$, and 
  \item $r: \GSConc(A_1), \ldots, \GSConc(A_n)\To \lit$ is a rule in  $R$. 
  \end{itemize}
With regard to argument $A$, the following holds:
$$
\begin{array}{l}
\GSConc(A) = \lit\\
\GSSub(A) = \GSSub(A_1), \ldots \GSSub(A_n), A\\
\GSTopRule(A) = r: \GSConc(A_1), \ldots,  \GSConc(A_n) \To_r \lit \\
\GSRules(A) = \GSRules(A_1), \ldots , \GSRules(A_n), \GSTopRule(A).
\end{array}
$$
We say that any arguments $\To_{\FACTS} \lit$ or $A_1, \ldots , A_n \To_r \lit$ are \emph{arguments for} $\lit$.
\end{definition}

The following standard definitions are from \cite{GovernatoriMAB04}.

\begin{definition}[Attack, support, and undercut]
A plain argument $A$ \emph{attacks} a plain argument $B$ if a conclusion of $A$ is the complement of
a conclusion of $B$. We define the \emph{attack relation} $\defeat$ such that, for any arguments $A$ and $B$, $\langle A, B\rangle\in \defeat$ (or, in short, $A\defeat B$) iff $A$ attacks $B$. A set of plain arguments $\Arg$ attacks a plain argument $B$ if there is an argument $A$ in $\Arg$ that attacks $B$. 

A proper subargument $B=\GSSub(A)$ of an argument $A$ is such that $B\not= A$.

An argument $A$ is \emph{supported} by a set of arguments $\Arg$ if every proper subargument of $A$ is in $\Arg$.

An argument $A$ is \emph{undercut} by a set of arguments $\Arg$ if $\Arg$ supports an argument $B$ attacking a proper subargument of $A$.
\end{definition}

Notice that conflicts between arguments only consider plain arguments: arguments of the form $\To_{\FACTS} \lit$ can be ignored because the set of facts is assumed to be consistent and no fact (or its negation) can occur in the head of any rule \cite{lpnmr2022}.
The definition above from \cite{GovernatoriMAB04} does not make any reference to the superiority relation, since it is easy to see that the current semantics is a special case of that of \cite{GovernatoriM00} when the superiority relation is empty (and for every argumentation theory can be transformed into an equivalent one without the superiority relation is empty). However, the superiority relation can be taken into account by incorporating it into the definition of attack. In other words, for any argument $A$, if $\GSTopRule(A) = r: \GSConc(A_1), \ldots,  \GSConc(A_n) \To_r \lit$,
$A$ attacks another argument $B$ if, and only if $\GSTopRule(A)$ is stronger than
$\GSTopRule(B)$.

We can now define the argumentation framework. 
\begin{definition}[Argumentation Framework]\label{def:argumentation_framework}
Let $D= (\FACTS, R, >)$ be an argumentation theory.
The argumentation framework $\ARG(D)$ determined by $D$ is
$( \AR, \defeat )$ where $\AR$ is the set of all
arguments constructed from $D$, and $\defeat$ is the attack relation defined above.
\end{definition}

\begin{definition}[Acceptable and rejected argument]\label{def:accept-reject}
Let $D= (\FACTS, R, >)$ be an argumentation theory and 
$\ARG(D)=( \AR, \defeat )$ be the argumentation framework determined by $D$. 
An argument $A$ in $\ARG(D)$ for $\phi$ is \emph{acceptable} w.r.t to a set of argument $\Arg$ in $\ARG(D)$ if $A$ is finite and every argument attacking $A$ is undercut by $\Arg$.

An argument $A$ is \emph{rejected} by sets of arguments $\Arg$ and $\Arg'$ in $\ARG(D)$ when a proper subargument $B$ of $A$ is in $\Arg$  or $B$  is attacked by an argument supported by $\Arg'$ .
\end{definition}

\begin{definition}[Sets of acceptable and rejected arguments]\label{def:accept-reject2}
Let $D= (\FACTS, R, >)$ be an argumentation theory and 
$\ARG(D)=( \AR, \defeat )$ be the argumentation framework determined by $D$. We define $J^D_i$ as follows.
\begin{itemize}
\item $J^D_0 = \emptyset$
\item $J^D_{i+1} = \{ A \in \AR \, | \, A \mbox{ is acceptable w.r.t. }
  J^D_i \}$ 
\end{itemize}
The set of {\em justified arguments} in an argumentation theory  $D$ is $\mathit{JArgs}^D = \cup_{i=1}^\infty J^D_i$.

We define $R^D_i$ as follows.
\begin{itemize}
\item $R^D_0=\emptyset$
\item $R^D_{i+1}=\{B\in \AR \, | \, \mbox{$B$ is rejected by }R^D_i \mbox{ and } \mathit{JArgs}^D\}$
\end{itemize}
The set of \emph{rejected arguments} in an argumentation theory $D$ is
$\mathit{RArgs}^D=\cup_{i=1}^{\infty} R^D_i$.
\end{definition}
$\mathit{JArgs}^D$ corresponds the \emph{extension of the argumentation framework} determined by $D$.

The following are thus standard result that can be obtained:
\begin{theorem}\label{thm:dung}
Let $D$ be an argumentation theory. Then, 
\begin{itemize}
\item an argument $A$ and its conclusion $\lit$ are justified w.r.t. the argumentation framework $\ARG (D)$ if, and only if  (a) $A\in \mathit{JArgs}^D$ and (b) $D\vdash +\partial \lit$;
\item an argument $A$ and its conclusion $\lit$ are rejected w.r.t. the argument framework $\ARG(D)$ is, and only if (a) $A\in \mathit{RArgs}^D$ and (b) and (b) $D\vdash -\partial \lit$.
\end{itemize}
\end{theorem}

\section{Stable Normative Explanations}\label{sec:TypesArgExpl}


We define the idea of \emph{normative explanation} for $\lit$, which is a normative decision or any piece of normative knowledge that  justifies $\lit$ and that is minimal \cite{lpnmr2022,GovernatoriORC22,LiuLRS22}.

\begin{definition}[Normative explanation]\label{def:expl}
Let $D= (\FACTS, R, >)$ be an argumentation theory and $\ARG(D)=( \AR, \defeat )$ be the argumentation framework determined by $D$.
The set $\Arg\subseteq \AR$ is a \emph{normative explanation} $\Case{\lit}{\ARG(D)}$ in $\ARG(D)$ for $\lit$ iff 
\begin{itemize}
 \item  $A\in \Arg$ is an argument for $\phi$ and $A$ is justified w.r.t. $\ARG(D)$;
 \item $\Arg$ is a minimal set in $\ARG(D)$ such that $A$ is acceptable w.r.t to $\Arg$. 
\end{itemize}
\end{definition}

\begin{example}\label{ex:running}
Suppose the law forbids engaging in credit activities without a credit license. 
Such activities are permitted for a person acting on behalf of another person (the principal), when the person is an employee of the principal, and the principal holds a credit license.  Some conditions are specified under which a person can be banned for credit activities. For example,  a person is banned if she becomes insolvent.
%
\begin{align*}
	R = \{ 	&    s_1 \colon \Rightarrow \neg \mathit{creditActivity},\\
      & s_2 \colon \mathit{creditLicense} \Rightarrow \mathit{creditActivity},\\
     & s_3 \colon \mathit{actsOnBehalfPrincipal}, \mathit{principalCreditLicense} \Rightarrow \mathit{creditActivity},\\
     & s_4 \colon \mathit{banned} \Rightarrow \neg \mathit{creditActivity},\\
    & s_5 \colon \mathit{insolvent} \Rightarrow \mathit{banned}\}\\
 	> = \{ & \langle s_2, > s_1\rangle, \langle s_3 > s_1\rangle, \langle s_4 > s_3\rangle, \langle s_4 > s_2\rangle\}. 
\end{align*}
Assume an argumentation theory $D = (\FACTS, R, >)$ where $F=\set{\mathit{insolvent}, \mathit{creditLicense}}$. Then, $\ARG(D)=(\AR , \defeat)$ is as follows:
%
\begin{align*}
	\AR = \{ 	&    A_1 \colon \Rightarrow_F \mathit{insolvent},\quad A_2 \colon  \Rightarrow_F \mathit{creditLicense},\\
     & A_3 \colon A_1 \Rightarrow_{s_5} \mathit{banned},\quad  A_4 \colon A_3 \Rightarrow_{s_4} \neg \mathit{creditActivity},\\
    & A_5 \colon A_2 \Rightarrow_{s_2} \mathit{creditActivity}\}\\
 	\defeat = \{ & \langle A_4, A_5\rangle \rangle\}. 
\end{align*}
It is easy to see that $\set{A_1, A_4} =\Case{\neg \mathit{creditActivity}}{\ARG(D)}$.
\end{example}

As discussed in \cite{lpnmr2022,GovernatoriORC22}, an explanation for a given normative conclusion $\lit$ is stable when adding new elements to that explanation does not affect its power to explain $\lit$. 



The following definition thus elaborates the ideas of \cite{GovernatoriORC22} for the argumentation setting of Section \ref{sec:Arg-DL}. 

\begin{definition}
Let $R$ a finite set of rules. 
We define the set of literals $\LIT (R)$ as $\{ \lit, \non\lit | \forall r\in R: \lit\in A(r) \text{ or }\non\lit\in A(r), R[\lit]\cup R[\non \lit]=\emptyset\}$. 

We write $\Arg_R$ to denote the set of all possible arguments that can be built from $R$ and any finite set $\FACTS$ of facts such that $\FACTS\subseteq \LIT (R)$. 
\end{definition}

\begin{definition}[Stable Normative Explanation]
\sloppy Let $\ARG(D)=( \AR, \defeat )$ be an argumentation framework determined by the argumentation theory $D = (\FACTS, R, >)$. We say that $\Arg=\Case{\lit}{\ARG(D)}$ is a \emph{stable normative explanation for $\lit$ in $\ARG(D)$} iff for all $\ARG(D')=( \AR', \defeat' )$ where $D' = (\FACTS', R, >)$ s.t. $\FACTS \subseteq \FACTS' \subseteq \LIT (R)$, we have 
that $\Arg=\Case{\lit}{\ARG(D')}$.
\end{definition}

\begin{example}
Let us consider the argumentation framework $\ARG(D)$ in Example \ref{ex:running}.
Then, $\set{A_1, A_4}$ is stable normative explanation for $\neg \mathit{creditActivity}$ in $\ARG(D)$, whereas, e.g., $\set{A_2, A_5}$ is not a stable normative explanation for $\mathit{creditActivity}$.
\end{example}

On the basis of Section \ref{sec:Arg-DL} and Theorem \ref{thm:dung} it is easy to verify that the computational results from \cite{lpnmr2022,GovernatoriORC22} hold also in this case (the proofs are similar and are omitted):

\begin{theorem}\label{thm:stable}
    Given an argumentation framework $\ARG(D)$ and a normative explanation, (a) the problem of determining if the explanation is stable is co-NP-complete and (b) the problem of determining if the explanation is not stable is NP-complete.
\end{theorem}




\section{From Argumentation to Deontic Logic}\label{sec:FromArgToModal}
Let us now show how to move from an argumentation setting to deontic logic.
\begin{intuition}\label{int:equivalence}
Let $D= (\FACTS, R,  >)$ be any argumentation theory and $\ARG(D)=( \AR, \defeat )$ be the argumentation framework determined by $D$. The relation between argumentation and deontic logic is based on the following intution:
\[
D\vdash +\partial \phi \text{ if and only if } \cM,w\models \OBL \lit \text{ for some world } w \text{ in some model }\cM.
\]
%
%
%
\end{intuition}

%
The idea behind it---the construction of a canonical model for a defeasible theory---was proposed in \cite{GovernatoriRC12} for a multi-modal variant of DL where modal operators (including obligations) were explicitly added in the language and proof theory. However, building an argumentation semantics for that formalism is particularly hard, as shown in \cite{GovernatoriRR18,GovernatoriRRV19}. 

We avoid those complexities and elaborate on the approach of \cite{GovernatoriRC12} by constructively stating that defeasible provability of any $\lit$ corresponds  to the obligatoriness of $\lit$, and---if $\PERM$ is the dual of $\OBL$---the non-provability of $\lit$ means that $\non\lit$ is permitted.

%


\subsection{Deontic Logic and Semantics}

Let us define our modal logic language and system.

\begin{definition}[Modal language and logic]\label{def:language_logic}
Let $\LIT$ be the set of literals of  our language $\mathcal{L}$. 
The language $\mathcal{L}(\LIT)$ of $\logic{E}_{\mathcal{L}}$ is defined as follows:
\begin{gather*}
 p ::= l \;|\; \neg p \;|\; \OBL \lit \;|\; \PERM \lit ,   
\end{gather*}
where $l$ ranges over $\PROP$ and $\lit$ ranges over $\LIT$. 

The logical system $\logic{E}_{\mathcal{L}}$ is based on $\mathcal{L}(\LIT)$ and is closed under logical equivalence.    
\end{definition}


\begin{proposition}\label{th:fragment-E}
The system $\logic{E}_{\mathcal{L}}$ is a fragment of system $\logic{E}$ \cite{chellas:1980}.
\end{proposition}

Given Proposition \ref{th:fragment-E}, we use neighbourhood semantics. However, we have to identify a proper subclass of frames and models. 

Let us first recall standard neighbourhood semantics. 
\begin{definition}[Frames and models]\label{def:neighbourhood-frame}
\sloppy A {\em neighbourhood frame} $\cF$ is a structure
$\tuple{W, \cN}$ where $W$ is a non-empty set of possible worlds and 
$\cN$ is a function $W\mapsto 2^{2^W}$.

\sloppy A {\em neighbourhood model} 
$\cM$ 
is obtained by adding an evaluation function $v: \PROP\mapsto 2^W$ 
to a neighbourhood frame.
\end{definition}

\begin{definition}[Truth in a model]\label{def:neigh-truth}
 Let $\cM$ be a model $\tuple{W, \cN, V}$ and
 $w\in W$. The truth of any formula $p$ in $\cM$ is defined inductively
 as follows:
\begin{enumerate}
\item standard valuation conditions for the boolean connectives;
\item $\cM,w\models \OBL \lit$ iff $\tset{\lit}\in
  \cN (w)$,
\item $\cM,w\models \PERM \lit$ iff $W- \tset{\lit}\not\in
  \cN (w)$.
\end{enumerate}
\end{definition}
\noindent
A formula $p$ is \emph{true at a world} in a model
iff $\cM,w\models p$; \emph{true in a model} $\cM$, written $\cM\models
p$ iff for all worlds $w\in W$, $\cM,w\models p$; \emph{valid in a
  frame} $\cF$, written $\cF\models p$ iff it is true in all models
based on that frame; \emph{valid in a class $\mathcal{C}$ of frames}, written $\mathcal{C}\models p$, iff it is valid in all 
frames in the class. Analogously, an inference rule $p_1 , \dots p_n \Rightarrow q$ (where $p_1 , \dots p_n$ are 
the premises and $q$ the conclusion) is valid in a class $\mathcal{C}$ of frames iff, for any $\cF\in\mathcal{C}$, if $\cF\models p_1 , \dots , \cF\models p_n$ then $\cF\models q$.

In order to introduce a semantics for our fragment, the following 
is needed.

\begin{definition}\label{def:dextension}
  Let $D= (\FACTS, R, >)$ be any argumentation theory, $\ARG(D)=( \AR, \defeat )$ be the argumentation framework determined by $D$, and $\LIT (D)$ be the set of literals 
occurring in $D$. The \emph{$D$-extension} $E(D)$ of a theory $D$ is the smallest set of
  literals such that, for all $\lit\in \LIT (D)$:
  \begin{enumerate}
    \item $\lit\in E(D)$ iff $\lit$ is justified w.r.t. $\ARG(D)$, 
    \item $\non \lit\in E(D)$ iff $\lit$ is not justified w.r.t. $\ARG(D)$.
  \end{enumerate}
\end{definition}

\begin{definition}\label{def:Lextension}
  Let $L$ be a consistent set of literals. A \emph{defeasible rule theory} is a structure $D=(R,>)$. 
  The \emph{$D$-extension of $L$} is the extension of the argumentation
  theory $(L,R,>)$; we denote it with $E_L (D)$.
\end{definition}

\begin{definition}[Dependency graph]\label{def:dependency-graph}
Let $D$ be any argumentation theory and $\LIT (D)$ be 
literals 
occurring in $D$.
The \emph{dependency graph} of $D$ is the directed graph $(V,A)$ where:
\begin{itemize}
  \item $V=\set{p \,|\, p\in\PROP, \set{p,\neg p} \cap
    \LIT(D)\neq\emptyset}$;
  \item $A$ is the set such that $(n,m)\in A$ iff 
  \begin{itemize}
    \item $n= \lit$ and $\exists r\in R[\lit]\cup
      R[\non \lit]$;
    \item $m=\psi$ and $\exists r\in R[\psi]\cup R[\non \psi]$ such
      that $\set{n,\non n}\cap A(r)\neq\emptyset$.
  \end{itemize}
\end{itemize} 
\end{definition}

\begin{proposition}\label{pro:consitentExtension}
  Let $L$ be a set of literals, $D=(R,>)$ be a defeasible rule theory such that the transitive closure of $>$ is acyclic and $D'=(L, R,>)$ be the corresponding argumentation theory such that the dependency graph of $D'$ is acyclic. Then, the $D$-extension of $L$ is consistent iff $L$ is
  consistent.
\end{proposition}

\begin{proof}
\label{pf:consitentExtension}
 The result is based on the proofs of Proposition 3.3 of \cite{Antoniou2001255}, and Theorem
 2 of \cite{GovernatoriR08-law} (see \cite{Billington93}). The proof shows that, given the fact that $>$ and the dependency graph of $D'$ are acyclic (which means that we do not have loops in the rules), then, if $D\nmdash \lit$ and $D\nmdash \non \lit$ then $\lit , \non\lit \in L$, which contradicts the assumption that $L$ is consistent.
\end{proof}

The definition of an appropriate structure considers an argumentation theory $D$:  (a) we add as worlds all  $D$-extensions without the empty set, while (b) to construct  neighbourhoods for each world, we build an $S^r$ relationship between possible worlds based on information in the rule $r$ for each rule in $D$ and ensure that the rule can actually be applied, and put together all $S^r$ relations.
%
%

\begin{definition}[Neighbourhood $D$-frame,  neighbourhood $D$-model, and truth]\label{def:D-frame-model}
Let $D=(\FACTS , R , >)$ be an argumentation theory such that the transitive closure of $>$ is acyclic and the dependency graph of $D$ is acyclic. A \emph{neighbourhood $D$-frame} is a structure 
$\tuple{W, \cN}$ where
\begin{itemize}
\item $W = \set{w \, | \, w\in (2^{E(D)}-\set{\emptyset})}$;
\item $\cN$ is a function with signature $W\mapsto 2^{2^{W}}$ defined as follows:
    \begin{itemize}
      \item $xS_{j}y$ iff $\exists r \in R$ such that $A(r)\subseteq x$ and $C(r)\in y$ 
          \item $\forall s\in R[\non C(r)]$ either
          \begin{enumerate}
             \item $\exists a\in A(s), a\notin x$; or
            \item $\exists t\in R[C(r)]$ such that $t>s$, $A(t)\subseteq x$
        \end{enumerate}
      \item $S_{j}(w)=\set{x\in W: wS_{j}x}$ 
      \item $\cS_{j}(w)=\bigcup_{C(r_{k})=
      C(r_{j})}S_{k}(w)$
      \item $\cN(w)=\set{\cS_{j}(w)}_{r_{j}\in R}$.
    \end{itemize}
\end{itemize} 
\sloppy A {\em neighbourhood $D$-model} 
$\cM$ 
is obtained by adding an evaluation function $v: \PROP\mapsto 2^W$ 
to a neighbourhood $D$-frame such that, for any $p\in \PROP$,  $v(p)=\set{w \, | \, p\in w}$.

\end{definition}


\begin{proposition}\label{th:class-D}
Let $C_{\cF}$, $C_{\cM}$, $C_{\cF_{D}}$ and $C_{\cM_{D}}$be, respectively, the classes of neighbourhood frames and models, and the classes of neighbourhood $D$-frames and $D$-models. Then, $C_{\cF_{D}}\subset C_{\cF}$ and $C_{\cM_{D}}\subset C_{\cM}$.
\end{proposition}

\subsection{Completeness}\label{sec:completeness}
To build canonical structures from an argumentation framework, we use defeasible rule theories by following Intuition \ref{int:equivalence} and Definitions \ref{def:dextension} and \ref{def:D-frame-model}. The construction considers all possible defeasible rule theories and, for each of them, all possible maximal consistent sets of facts that can be generated. In a nutshell, the procedure 
runs as follows:

\begin{enumerate}
\item \textbf{Considering all defeasible rule theories of the language.} Given the language $\mathcal{L}$, the set of all defeasible rule theories is $\mathcal{D}$. 
\item \textbf{Constructing worlds.}
For each defeasible rule theory $D\in \mathcal{D}$, add as worlds all maximal consistent sets of formulae containing all $D$-extensions of each $L\in 2^{\LIT (D)}$) plus the negation of all literals that do not occur in $D$.
\item \textbf{Constructing   neighbourhoods for each world.} Proceed as in Definition \ref{def:D-frame-model}. 
\end{enumerate}
%


\begin{definition}[$\mathcal{L}$-maximality]
A set $w$ is $\mathcal{L}(\LIT)$-maximal iff for any formula $p$ of $\mathcal{L}(\LIT)$, either $p \in w$, or $\neg p \in w$.
\end{definition}

\begin{lemma}[Lindenbaum's Lemma]
Let $D$ any defeasible rule theory. Any consistent set $w_{E_L(D)}$ of formulae in the language $\mathcal{L}(\LIT)$ consisting of a $D$-extension of any $L$ can be extended to a consistent $\mathcal{L}(\LIT)$-maximal set $w_{E_L(D)}^+$.

\end{lemma}

\begin{proof}
Let $p_1, p_2, \ldots$ be an enumeration of all the possible formulae in $\mathcal{L}(\LIT)$.
\begin{itemize}
\item
$w_0 := w_{E_L(D)}$;
\item
$w_{n+1} := w_n \cup \{ p_n \}$ if its closure under the axioms and rules of a given logic \Lg{S} is consistent, $w_n \cup \{ \neg p_n \}$ otherwise;
\item
$w_{E_L(D)}^+ := \bigcup_{n\geq0} w_n$.
\end{itemize}
\end{proof}

\begin{definition}[Canonical neighbourhood $D$-model]\label{def:defeasibleNeighbourhood}
Given the language $\mathcal{L}$, let $\mathcal{D}$ be the set of all defeasible rule theories that can be obtained from $\mathcal{L}$. For all $D_i=(R_i,>_i)\in \mathcal{D}$ the \emph{canonical neighbourhood model}
  is the structure $\mathcal{M}_{\mathcal{D}} =   (W,  \cN,v)$
  where
  \begin{itemize}
    \item $W = \bigcup_{\forall D_i\in\mathcal{D}} W_i$ where $W_i = \set{w_{L} \, | \, \forall L\in 2^{\LIT (D_i)}, w_{L}= w_{E_L(D)}^+}$.
  %
    \item $\cN$ is a function with signature $W\mapsto 2^{2^{W}}$ defined as follows: 
    \begin{itemize}
      \item $xS^{i}_{j}y$ where $\OBL \lit\in x$ iff $\exists r\in R_i$ such that $C(r)=\lit$, $A(r)\subseteq x$ and $C(r)\in y$ where $x,y\in W_i$;
          \item $\forall s\in R_{i}[\non C(r)]$ either
          \begin{enumerate}
            \item $\exists a\in A(s), a\notin x$; or
            \item $\exists t\in R_{i}[C(r)]$ such that $t>s$, $A(t)\subseteq x$
        \end{enumerate}
      \item $S^{i}_{j}(w)=\set{x\in W_i: wS^{i}_{j}x}$, 
      \item $\cS^{i}_{j}(w)=\bigcup_{C(r_{k})=
      C(r_{j})}S^{i}_{k}(w)$,
      \item $\cN(w)=\set{\cS^{i}_{j}(w)}_{r_{j}\in R_{i}}$;
    \end{itemize}
  \item for each $\lit\in\LIT$ and any $w\in W$, $v$ is an evaluation function such that $w\in v(\lit)$ iff $\lit\in
    w$, and $w\not\in v(\lit)$ iff $\non \lit\in w$.
  \end{itemize}
\end{definition}

\begin{lemma}[Truth Lemma]\label{truth_lemma}
If $\cM=(W, \cN, v )$ is canonical for \Lg{S}, where $\Lg{S} \supseteq \Lg{E}_{\mathcal{L}}$, then
for any $w\in W$ and for any formula $p$, $p \in w$ iff $\cM , w\models \lit$.
\end{lemma}

\begin{proof} 
The proof 
is by induction on the length of an expression $p$. We consider only some relevant cases.

Assume $p$ is a literal $\lit$. If $\lit \in w$, by the semantic evaluation clause it holds that  $\cM, w\models \lit$. For the opposite direction, assume that $\cM, w\models\lit$, by construction $\lit \in w$.

If, on the other hand, $p$ has the forms
$\OBL \lit$ and $\PERM \lit$, and $p\in w$, then, by construction (respectively), $\|\lit\| \in \cN (w)$ and $W- \tset{\lit}\not\in
  \cN (w)$. By definition $\cM,w\models \OBL \lit$ and $\cM,w\models \PERM \lit$, respectively. Conversely, if $\cM, w\models \OBL \lit$ and $\cM, w\models \PERM \lit$, then $\| \lit \| \in \cN (w)$ and $W- \tset{\lit}\not\in
  \cN (w)$, and by construction of $\cN$, $\OBL \lit \in w$ and $\PERM \lit \in w$.
\end{proof}

The canonical model exists, it is not empty, and it
is a neighbouhood $D$-model. 
Consider any formula $p \notin \Lg{S}$ such
that $\Lg{S}\supseteq \logic{E}_{\mathcal{L}}$; $\{\neg
p \}$ is consistent and it can be extended to a maximal set $w$ such that for
some canonical model, $w \in W$. By Lemma \ref{truth_lemma}, $w\not\models
p$.

\begin{corollary}[Completeness of $\logic{E}_{\mathcal{L}}$]
The systems $\logic{E}_{\mathcal{L}}$  is sound and complete with respect to the class of neighbourhood  $D$-frames.
\end{corollary}


\begin{corollary}\label{thm:corollary}
  Let $\cM$ be any neighbourhood $D$-model. Then
  \begin{enumerate}
     \item $\cM \models \OBL \lit$ iff there exists an argumentation theory $D=(\FACTS,R,>)$ such that $\lit$ is justified w.r.t. $\ARG(D)$;
    \item $\cM \models \PERM \lit$ iff there exists an argumentation theory $D=(\FACTS,R,>)$ such that $\lit$ is not justified w.r.t. $\ARG(D)$.
  \end{enumerate}    
\end{corollary}

\section{Stable Explanations in Neighbourhood Semantics}\label{sec:deontic_explanation}
The definition of normative explanation of Section \ref{sec:TypesArgExpl} can be appropriately captured in our deontic logic setting.
First of all, we have to formulate the modal version of an argument. 

\begin{proposition}[Neighbourhood $D$-model for an argument]\label{th:A-model}
Let $D= (\FACTS, R, >)$ be an argumentation theory, $\ARG(D)=( \AR, \defeat )$ be the argumentation framework determined by $D$, and $\cM_{D} =( W, \cN, v)$ be the corresponding neighbourhood $D$-model. 
Argument $A\in \AR$, where $\GSConc (A)=\psi$, is justified w.r.t. $\ARG(D)$ iff, if $\GSSub(A)=\set{A_x\, | \, \forall x\in\set{1, \dots , s\, | \, s\geq 1}, \GSConc (A_x)=\lit_{i_x}}$, then the following condition holds in $\cM_D$: 
%
%
%
\begin{equation*}
\exists w_1 \dots \exists w_s\in W\left\{
\begin{array}{ll}
\forall i_2\in\set{1_2,\dots n_2}, (\forall i_1\in\set{1_1,\dots n_1},\tset{\lit_{i_{1}}}\in\cN(w_1)
\Rightarrow \\
\Rightarrow \tset{\lit_{i_{2}}}\in\cN(w_1))\\
\& ~ \\
\forall i_3\in\set{1_3,\dots n_3}, (\forall i_2\in\set{1_2,\dots n_2},w_2\in\tset{\lit_{i_{1}}}
\Rightarrow \\
\Rightarrow \tset{\lit_{i_{2}}}\in\cN(w_2))\\
\& ~ \\
\;\vdots\\
\&~\\
\forall i_s\in\set{1_s,\dots n_s}, (\forall i_{s-1}\in\set{1_{s-1},\dots n_{s-1}},w_{s-1}\in\tset{\lit_{i_{s-1}}}
\Rightarrow \\
\Rightarrow \tset{\lit_{i_{s}}}\in\cN(w_{s-1}))\\
\&~\\
\forall i_s\in\set{1_s,\dots n_s}, (w_{s}\in\tset{\lit_{i_{s}}}
\Rightarrow \tset{\psi}\in\cN(w_{s}))
\end{array} \right.
\end{equation*}
The model $\cM_D$ is called a \emph{neighbourhood $D$-model for $A$}.
%
\end{proposition}

\begin{proof}

\noindent ($\Rightarrow$) 
\sloppy The argument $A$ has either the form $\To_{\FACTS} \psi$ or the form 
$A_{1_s}, \ldots , A_{n_s} \To_r \psi$, where $A_{i_s}$, $1_s\leq i_s \leq n_s$, is an argument constructed from $D$, and $r: \GSConc(A_{1_s}), \ldots, \GSConc(A_{n_s})\To \psi$ is a rule in  $R$. 

{\em Case (1)}. If $A: \To_{\FACTS} \psi$, by construction of $\cM_D$ (see Definition \ref{def:D-frame-model}) $s=1$ and there is a world $w_s$ such that $\tset{\psi}\in\cN (w_s)$.

{\em Case (2)}. If $A_{1_s}, \ldots , A_{n_s} \To_{r_s} \psi$, then, by construction of $\cM_D$, there a world $w_s$ such that for each $\GSConc(A_{i_s})$, $w_s\in \tset{\GSConc(A_{i_s})}$ and $\tset{\psi}\in \cN(w_s)$. 
Consider now each $A_{i_s}$, having the form $A_{1_{s-1}}, \ldots , A_{n_{s-1}} \To_{r_{s-1}} \GSConc(A_{i_s})$, which in turn can fall within Case (1) or Case (2). Suppose, for example, that each $A_{i_{s-1}}$ falls within Case (2). Then, by construction of $\cM_D$, there a world $w_{s-1}$ such that for each $\GSConc(A_{i_{s-1}})$, $w_{s-1}\in \tset{\GSConc(A_{i_{s-1}})}$ and $\tset{\GSConc(A_{i_{s}})}\in \cN(w_{s-1})$. Similarly, for the other cases.

\sloppy Since, $A$ is finite, it means that there are sub-arguments having the form $A_{1_1}, \ldots , A_{n_1} \To_{r_1} \lit_{i_2}$ to which we can develop a similar argument.

\noindent ($\Leftarrow$) 
The proof for this direction runs similarly as the one for the case ($\Rightarrow$).
\end{proof}

The concept of normative explanation directly follows from Proposition \ref{th:A-model}.

\begin{proposition}[Neighbourhood $D$-model for a normative explanation]\label{th:model_case}
Let $D= (\FACTS, R, >)$ be an argumentation theory, $\ARG(D)=( \AR, \defeat )$ be the argumentation framework determined by $D$, and $\cM_{D} =( W, \cN, v)$ be the corresponding neighbourhood $D$-model. 

If $\Case{\psi} {\ARG(D)}=\set{A_1 , \dots , A_n}$ then $\cM_{D}$ is neighbourhood $D$-model for each argument $A_k$, $1\leq k \leq n$.  

The model $\cM_D$ is called a \emph{neighbourhood $D$-model for $\Case{\psi}{\ARG(D)}$}.
\end{proposition}

We can semantically isolate the arguments in a normative explanation by using Proposition \ref{th:A-model} as well as 
by resorting to the notion of generated sub-model \cite{Benthem_Pacuit,Pacuit2017-PACNSF}.

\begin{definition}[Generated submodel \cite{Benthem_Pacuit,Pacuit2017-PACNSF}]
\sloppy Let $\cM = (W,  \cN,v)$ be any neighbourhood model. A generated submodel $\cM_X = (X,  \cN_X, v_X)$ of $\cM$ is neighbourhood model where $X\subseteq W$, $\forall Y\subseteq W,\forall w\in X, Y\in\cN(w) \Leftrightarrow Y\cap X\in \cN_X(w)$. 
\end{definition}

\begin{proposition}
[Generated $D$-submodel for a normative explanation]\label{th:submodel_case}
Let $D= (\FACTS, R, >)$ be an argumentation theory, $\ARG(D)=( \AR, \defeat )$ be the argumentation framework determined by $D$, $\mathcal{X}=\Case{\psi}{\ARG(D)}$, $\cM_{D} =( W, \cN, v)$ be a neighbourhood $D$-model for $\mathcal{X}$, and $\cM_{D_{\mathcal{X}}} = (W_{\mathcal{X}},  \cN_{\mathcal{X}}, v_{\mathcal{X}})$ be a generated submodel of $\cM_{D}$.

$\mathcal{X}=\set{A_1 , \dots , A_n}$ iff $W_{\mathcal{X}} = W - X$ where
\begin{gather*}
X=\set{w \, |\, w\in W, \,\forall \lit\in w:\lit\in F\, \&\, A_x\in \AR,\, A_x\not\in\mathcal{X}\text{ and }A_x:\To_{F}\lit}
\end{gather*}   
The model $\cM_{D_{\mathcal{X}}}$ is called the {\em generated $D$-submodel for $\mathcal{X}$}. 
\end{proposition}
\begin{proof}[Skecth]
The model $\cM_{D_{\mathcal{X}}}$ is the generated submodel obtained by isolating in $\cM_{D}$ precisely those worlds, and only those worlds in which the factual literals and factual arguments are those which are needed in $\mathcal{X}$. Hence, (1) if we consider them, all arguments in $\mathcal{X}$ are justified; (2) if we consider only the arguments in $\mathcal{X}$, then, by construction (see Proposition \ref{th:A-model}), $W_{\mathcal{X}}$ contains all worlds in $W$ except those in which facts not needed in $\mathcal{X}$ are the case. Notice that, since $\cM_{D_{\mathcal{X}}}$ is a generated submodel, the truth values of modal formulae are preserved. 
\end{proof}


The semantic reconstruction of stable normative explanation thus trivially follows. 

\begin{corollary}[Stable normative explanation in neighbourhood $D$-models]\label{th:arg-contr2}
Let $D= (\FACTS, R, >)$ be an argumentation theory and $\ARG(D)=( \AR, \defeat )$ be the argumentation framework determined by $D$. 

If $\mathcal{X}=\Case{\psi} {\ARG(D})=\set{A_1 , \dots , A_n}$ is a stable normative explanation for $\psi$ in $\ARG(D)$ and $D^+= (\FACTS^+, R, >)$ is the argumentation theory where $\FACTS^+ = \{ \lit | \forall r\in R: \lit\in A(r) \text{ and } R[\lit]\cup R[\non \lit]=\emptyset\}$, then $\Case{\psi} {\ARG(D^+})$, and $\cM_{D_{\mathcal{X}}} = \cM_{D^+_{\mathcal{X}}}$ such that $\cM_{D_{\mathcal{X}}}$  and $\cM_{D^+_{\mathcal{X}}}$ are, respectively, the generated $D$-submodel and generated $D^+$-submodel for $\mathcal{X}$.
\end{corollary}
In other words, a stable explanation considers a neighbourhood model where all possibile facts of a theory $D$ are the case and requires that in such a model the conclusion $\psi$ is still justified.

\section{Summary}
In this paper we investigated the concept of stable normative explanation in argumentation, which was elsewhere introduced in Defeasible Logic using proof-theoretic methods. Then we have devised in a deontic logic setting a new method to construct appropriate neighborhood models from argumentation frameworks and we have characterised accordingly the notion of stable normative explanation. The problem of determining a stable normative explanation for a certain legal conclusion means to identify a set of facts, obligations, permissions, and other normative inputs able to ensure that such a conclusion continues to hold when new facts are added to a case. This notion is interesting from a logical point of view---think about the classical idea of inference to the best explanation---and we believe it can also pave the way to develop symbolic models for XAI when applied to the law.

The idea of stability, since it requires to consider adding new inputs, can be reexamined through the revision of the given argumentation theory.
Formally, given an initial argumentation theory $D_{init}$, the revised theory $D$, and the target conclusion $\phi$, we could formally define change operations as follows: \begin{description}
    \item[Expansion:] from $D_{init}\not\vdash \phi$ to $D\vdash \phi$.
    \item[Contraction:] from $D_{init}\vdash \phi$ to $D\not\vdash \phi$.
    \item[Revision:] from $D_{init}\vdash \phi$ to $D\vdash \non \phi$.
\end{description}

How such an intuition can be fully exploited in the context of the current research is left to future research.


\bibliographystyle{splncs04}
\bibliography{references}
\end{document}